\documentclass[10pt,a4paper]{article}
\usepackage[utf8]{inputenc}
\usepackage[T1]{fontenc}
\usepackage{amsmath}
\usepackage{amsfonts}
\usepackage{amssymb}
\usepackage{graphicx}
\usepackage{graphicx}
\usepackage[a4paper, total={6in, 8in}]{geometry}
\usepackage{amssymb,array}
\usepackage{amsmath}

\usepackage[all]{xy}
\usepackage{wrapfig}

\usepackage{microtype}
\usepackage{graphicx}
\usepackage{booktabs} 
\usepackage{multirow}   
\usepackage{caption}
\usepackage{subcaption}
\usepackage{hyperref}
\usepackage{multirow,hhline,array}
\usepackage{nicefrac,xfrac}

\usepackage{multirow}
\usepackage{pifont}

\usepackage{todonotes}
\usepackage{apxproof}

\usepackage{amsthm}
\newtheorem{theorem}{Theorem}
\newtheorem{lemma}{Lemma}
\newtheorem{corollary}{Corollary}
\theoremstyle{remark}
\newtheorem{example}{Example}
\newtheorem{remark}{Remark}
\theoremstyle{definition}
\newtheorem{definition}{Definition}

\newcommand{\N}{\mathbb{N}}
\newcommand{\R}{\mathbb{R}}
\renewcommand{\P}{\mathbb{P}}
\renewcommand{\d}{\textnormal{d}}
\newcommand{\I}{\mathbb{I}}
\newcommand{\E}{\mathbb{E}}
\renewcommand{\O}{\mathcal{O}}
\renewcommand{\H}{\mathcal{B}}
\newcommand{\Hypo}{\mathcal{H}}
\newcommand{\D}{S}

\newcommand{\Norm}{\mathcal{N}}

\newcommand{\X}{\mathcal{X}}
\newcommand{\T}{\mathcal{T}}

\usepackage{csquotes}
\usepackage{hyperref}

\begin{document}
\title{
Suitability of Different Metric Choices for Concept Drift Detection\thanks{We gratefully acknowledge funding by the BMBF TiM, grant number 05M20PBA.}}
%
%
\author{%
	Fabian Hinder \\
	Cognitive Interaction Technology (CITEC)\\
	Bielefeld University\\
	Inspiration 1, D-33619 Bielefeld, Germany \\
	\texttt{fhinder@techfak.uni-bielefeld.de} \\
	\and
	Valerie Vaquet \\
	Cognitive Interaction Technology (CITEC)\\
	Bielefeld University\\
	Inspiration 1, D-33619 Bielefeld, Germany \\
	\texttt{vvaquet@techfak.uni-bielefeld.de} \\
	\and
	Barbara Hammer \\
	Cognitive Interaction Technology (CITEC)\\
	Bielefeld University\\
	Inspiration 1, D-33619 Bielefeld, Germany \\
	\texttt{bhammer@techfak.uni-bielefeld.de} \\
}
\maketitle              
\begin{abstract}
The notion of concept drift refers to
the phenomenon that the distribution, which is underlying the observed data,
changes over time;
as a consequence machine learning models may become inaccurate and need adjustment.
Many unsupervised approaches for drift detection rely on measuring the discrepancy between the sample distributions of two time windows. This may be done directly, after some preprocessing (feature extraction, embedding into a latent space, etc.), or with respect to inferred features (mean, variance, conditional probabilities etc.). Most drift detection methods can be distinguished in what metric they use, how this metric is estimated, and how the decision threshold is found.
In this paper, we analyze structural properties of the drift induced signals in the context of different metrics. We  compare different types of estimators and metrics theoretically and empirically and investigate the relevance of the single metric components. In addition, we propose  new choices and demonstrate their suitability in several experiments.
\end{abstract}

\section{Introduction}

One popular assumption in classical machine learning is that the observed data is generated i.i.d.\ according to some unknown underlying and stationary probability $\P_X$. 
Yet, stationarity is  often violated 
for realistic learning tasks 
 such as machine learning based on (streaming) social media entries or measurements of IoT devices, which are subject to continuous change
\cite{DBLP:journals/adt/BifetG20,DBLP:journals/widm/TabassumPFG18a}.
Here, concept drift, i.e.\ changes of the underlying distribution $\P_X$ occurs,  caused e.g.\ by seasonal changes, changed demands, ageing of sensors, etc.
Learning with drift can be dealt with in  
different ways. Often, data are treated via windowing techniques, and the model is continuously adapted based on the characteristics of the data in an observed time window. Thereby, many approaches deal with supervised scenarios and they aim for a small interleaved train-test error. In recent years, some approaches  deal with concept drift in unsupervised settings \cite{DBLP:journals/cim/DitzlerRAP15,Lu_2018}.
One fundamental problem, which is part of supervised learning schemes as well as 
unsupervised drift modelling and which will be in the focus of this publication, is the challenge of drift detection and
determination of the time point when drift occurs. 

According to~\cite{Lu_2018} most drift detection schemes proceed in four stages: 1)~collecting data, 2)~building a descriptor of the data in two time windows, 3)~computing a similarity based on the obtained the descriptor, 4)~normalize the similarity, e.g. by considering an appropriate statistical test. 
This work focuses on the second and third stage of this scheme, which constitute the most crucial ones. The first stage can be solved in many problem-specific ways without a major effect on the next stages. The decision process in stage four can be bounded independently of the concrete realization:  the difference of the output of stage three under the null hypothesis (no drift) and the alternative (drift) constitutes such a bound.

The aim of the present work is to determine the influence of the two major ingredients of stage 2 and 3, namely the used descriptor (stage 2) and the similarity measure applied to the descriptor (stage 3) and to evaluate their influence on the capability to detect drift and localize it in time.
 We will empirically show that the chosen similarity measure is  of minor importance. The descriptor has an impact. In lay terms, it is more important how to estimate
 rather than what to estimate. This claim will be investigated from a theoretical and an empirical perspective using different estimation schemes.

Beyond this general comparison, we provide a new method to construct dataset-specific models to solve stages two and three in an efficient way: random projection-based and moment tree-based binning. This is of particular interest since dataset-agnostic similarity measures face the challenge of an inherent trade-off between decision accuracy and convergence speed.

This work is structured as follows: first (Section~\ref{sec:setup}) we recall relevant work from the literature and define the problem setup -- in particular, we describe different approaches to tackle the four stages (Section~\ref{sec:overview}). We also provide a general argument when an estimator is capable of drift detection (see Theorem~\ref{thm:universal}). In the last section (Section~\ref{sec:exp}) we evaluate the metrics and estimators -- showing their strengths and weaknesses --  and show the suitability of our proposed approaches. 

\section{Problem Setup}
\label{sec:setup}

In the usual time invariant setup of machine learning, one considers a generative process $\P_X$, i.e. a probability measure, on the data space $\X$. In this context, one views the realizations of $\P_X$-distributed, independent random variables $X_1,...,X_n$ as samples.
Depending on the objective, learning algorithms try to infer the data distribution based on these samples, or, in the supervised setting, the posterior distribution. We will not distinguish between these settings and only consider distributions in general,
subsuming supervised and unsupervised modeling \cite{DBLP:journals/corr/WebbLPG17}.

Many processes in real-world applications are time dependent, so it is reasonable to incorporate time into our considerations. One prominent way to do so, is to consider an index set $\T$, representing time, and a collection of probability measures $p_t$ on $\X$, indexed over $\T$, which may change over time \cite{asurveyonconceptdriftadaption}. We will usually assume $\T = [0,1]$.
In the following, we investigate the relationship of those $p_t$.
Drift refers to the fact that $p_t$ varies for different time points, i.e.\ 
\begin{align*}
    \exists t_0,t_1 \in \T : p_{t_0} \neq p_{t_1}.
\end{align*}
In this context, we consider a sequence of samples $(X_1,T_1),(X_2,T_2),...$, with $X_i \sim p_{T_i}$ and $T_i \leq T_{i+1}$, as a stream.
Notice, that we will usually use the shorthand drift instead of concept drift.

In this contribution we will mainly focus on the case of one single abrupt drift, i.e. there exist probability measures $P$ and $Q$ and a time point $t_0 \in \T$, such that
\begin{align*}
    p_t = \begin{cases}
    P, & t \leq t_0 \\
    Q, & t > t_0
    \end{cases}.
\end{align*}
In this context we can ask two questions, which are referred to as drift detection:
\begin{enumerate}
    \item \emph{Whether} there is drift, i.e. does $P \neq Q$ hold? \label{Q:Whether}
    \item If so, \emph{when} does the drift occur, i.e. what is $t_0$? \label{Q:When}
\end{enumerate}

\subsection{A General Scheme for Drift Detection}


As most drift detection methods are applied in a streaming context, one usually considers time-dependent data samples $\D(t)$, observed during a time period $W(t)$.
To detect drift, one estimates the similarity of the distributions of a (presumably before drift) \emph{reference time-interval} (or window) $W_-(t)$ and a \emph{current time-interval} $W_+(t)$, which are obtained by splitting $W(t)$. The estimation is done using the sub-samples  $\D_-(t)$ and $\D_+(t)$ called \emph{windows} of $\D(t)$ that correspond to $W_-(t)$ and $W_+(t)$, respectively.
The way this is done varies depending on the specific algorithm.
In this section we  discuss some of the most prominent choices for the relevant stages 1-4 of this drift detection scheme as described in \cite{Lu_2018}. 

\paragraph{Stage 1: Acquisition of data:} As stated above most approaches are based on sliding windows, however, the concrete implementation can vary. In particular, the reference window is realized in different ways: as sliding window, stationary, growing window, implicitly within a model, etc.
To illustrate the idea we describe the examples of (a variant of) ADWIN and a simple version of an implicit reference window:

\begin{example}
ADWIN~\cite{adwin} uses only one sliding window $\D(t_\text{now})$. 
To test for drift this window is split successively into to halves,  $\D_-(t;t_\text{now})$ and $\D_+(t;t_\text{now})$. 
Then, these  are compared using a suitable distance measure $\hat d$, i.e. the statistic of ADWIN is given by 
$\sup_{t} \hat{d}(\D_-(t;t_\text{now}),\D_+(t;t_\text{now}))$.
In the original version ADWIN ``prepocesses'' the data by comparing the result of a fixed classification model against reference labels. However, extensions with other statistical tests are straightforward.
\end{example}

\begin{example}
A simple approach with implicit reference window consists of a reference mean $\hat{\mu}_\text{ref} = \mu(S_-(t))$ 
and a sliding window $\D_+(t)$ of fixed size corresponding to $W_+(t)$. If there is no drift, the mean in the current window and the reference should be the same, i.e. $\mu(S_+(t)) \approx \hat{\mu}_\text{ref}$. Based 
in this assumption a drift detection can be performed using a $t$-test. 
Once a sample drops out of the current window $S_+(t)$ it is used to update $\hat{\mu}_\text{ref}$.
\end{example}

Apart from these examples, some approaches use preprocessing such as a deep latent space embedding. We do not explain those possibilities in more detail.
Instead, we focus on the case of two windows only and try to evaluate the suitability of different distance measures for the task at hand.

\paragraph{Stage 2: Building a descriptor:} 
\begin{figure}[t]
    \centering
    \begin{align*}
    \xymatrix{
            & \H \ar[dr]^s
            & \\
\bigcup_n ( \T \times \X )^n \times \T \ar[rr]_{\hat{d}} \ar[ru]^A
&& \R
}
\end{align*}
    \caption{Two stage scheme of estimating distribution similarity from data. Distance of distributions ($\hat{d}$) can be estimated by building a descriptor ($A$) and then computing a similarity ($s$).}
    \label{fig:process}
\end{figure}
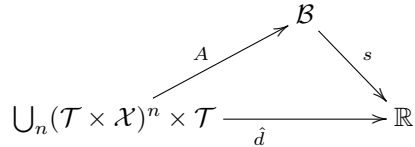
Comparing two distributions directly based on a sample is usually complicated. Therefore, the process is split into two parts which are visualized in Fig.~\ref{fig:process}: First a descriptor of the distributions is built (this corresponds to $A$ in Fig.~\ref{fig:process}), and then the similarity of the distribution is computed based on that descriptor ($s$ in Fig.~\ref{fig:process}). Possible descriptors are grid- or tree-based binnings, neighbor-, and kernel approaches. We list some of the most popular descriptors together with suitable similarity measures in Section~\ref{sec:overview}. 

\paragraph{Stage 3: Computing similarity:} As stated in the last paragraph, computing the similarity of two samples is often reduced to a comparison of descriptors which are based on those samples ($s$ in Fig.~\ref{fig:process}). 
Although, several approaches for building descriptors exist, many 
admit the same or at least comparable similarity measures. For example, if we consider binning descriptors, it does not matter whether the bins are obtained from a grid or a tree, or if the grid or tree is adjusted to the presented data or not. 

\paragraph{Stage 4: Normalization:} As the obtained similarities typically depend on both, the method, i.e. stages 1-3, and also the concrete distribution at hand, it is necessary to normalize the result  to obtain a useful scale. One of the most common ways to do this is by a relation of the similarity to the statistic of a statistical test; in this case the $p$-value offers a normalized scale. In the literature a large variety of approaches are considered. 
However, independently of the concrete normalization, the presence of drift can be observed from the output of stage 3; 
more formally the post hoc optimal normalization after stage 3 provides an upper bound on the quality of any concrete normalization. 
Therefore, we will focus on the output after stage 3 in the following.

\paragraph{Beyond stage 4: Ensemble and hierarchical approaches:}
Some authors~\cite{Lu_2018} suggest to combine multiple drift detectors. They are usually arranged in an ensemble, e.g. by combining multiple $p$-values after stage 4 into a single one, or hierarchical, e.g. by combining a computationally inexpensive but imprecise detector with a precise but computationally expensive validation. Although, those approaches differ on a technical level, they do not from a theoretical perspective, as the suggested framework is sufficiently general.

\subsection{Formal Setup and Research Question}

Before we can formally specify question~\ref{Q:Whether} and \ref{Q:When}, we first have to define the sampling process:
\begin{definition}
Let $\X$ be a data space and $\T \subset \R$.
Let $(p_t,P_T)$ be a drift process \cite{davidd} on $\X$ and $\T$, i.e. a distribution $P_T$ on $\T$ and a Markov kernel $p_t$ from $\T$ to $\X$. 
A \emph{window} $\D$ drawn from $p_t$ during a time interval $W \subset \T$ is a sample $\D = \{(x_1,t_1),\cdots,(x_n,t_n)\}$ drawn i.i.d. from $p_t P_T[\:\cdot\:|\:W\:]$, assuming $P_T(W) > 0$.
\end{definition}

We use the following notation:
If the choice of $W$ is not specified, we will assume $W = \T$. 
For the sub-intervals $W_-(t) := (-\infty,t] \cap W$ and $W_+(t) := (t,\infty) \cap W$, we define the sub-windows $\D_-(t) := \{(x',t') \in \D \:|\: t' \in W_-(t)\}$ and $\D_+(t)$ analogously.  


\paragraph{Question~\ref{Q:Whether}: ``Whether'':}
It was shown in \cite[Theorem 2]{conttime} that drift is equivalent to different sub-window distributions, i.e. it exists a $t \in W$ such that $p_{W_-(t)} \neq p_{W_+(t)}$, here $p_W$ denotes the distribution during $W$. Since we do not observe the underlying distributions, but only a window $\D$, it is reasonable to quantify this using estimating of the distance
\begin{align*}
    \hat{d}(\D_-(t),\D_+(t)) := (s \circ A)(\D,t) 
\end{align*}
which should be (significantly) larger than 0 if and only if there is drift. Here we decompose $\hat{d}$ as described before into a descriptor $A$ and a similarity $s$.
%
Control of the uncertainty of the sampling process when detecting drift can be formalized as follows:
\begin{definition}
%
%
Let $(p_t,P_T)$ be a drift process, and $\D$ denote a window drawn from it.
An \emph{estimator $(A,s)$} is a pair of measurable maps, one mapping windows to descriptors, i.e. \mbox{$A : \cup_n (\X \times \T)^n \times \T \to\H$,} the other mapping descriptors to similarities, i.e. $s : \H \to \R$. We refer to $\H$ as the \emph{description space}.

An estimator is \emph{drift detecting}, iff it raises correct alarms with a high probability in the following sense: There exists a $0 < \delta < 1/2$ and a number $n$ such that with probability at least $1-\delta$ over all choices of $\D$, with $|\D| > n$, it hold $s \circ A(\D,t) > 0$ for some $t$ if and only if there is drift. 

An estimator is \emph{surely drift detecting}, iff it raises correct alarms with arbitrarily high certainty, that is the above statement holds for all $0 < \delta < 1$. 
\end{definition}

Notice, that this definition is applicable for general drift, including gradual, incremental, and periodic.
Furthermore, the difference between drift detection and  sure drift detection only occurs in the limit of the size of $\D$. As long as we are restrained to windows of fixed sample size, both notions are effectively the same.

\paragraph{Question~\ref{Q:When}: ``When'':} We are  interested in finding the time point $t_0$  where the drift actually occurs. This is often estimated by  the point $\hat{t_0}$ with largest difference of the sub-windows, i.e. 
\begin{align*}
    \hat{t_0} = \underset{t \in \T}{\textnormal{ arg max }} \hat{d}({\D_-(t)}, {\D_+(t)}).
\end{align*}
The precision of this estimator can be quantified by mean ratio of samples between the true drift event $t_0$ and its estimate $\hat{t_0}$. 
This can be captured in the following definition:
\begin{definition}
Let $(p_t,P_T)$ be a drift process with a single abrupt drift event at $t_0$ with $0 < P_T(W_-(t_0)) < 1$. Let $\D$ be a window drawn from $p_t$.

We define the \emph{precision} as $1-P_T( \; [t_0,\hat{t_0}) \cup (\hat{t_0},t_0]\; )$.\footnote{Recall that $[a,b) = (a,b] = \emptyset$ for $a \geq b$.} 

We say that an estimator $(A,s)$ is \emph{precise}, iff for all $0 < \delta < 1$ and $\epsilon > 0$ there exists a number $n$, such that with at least probability $1-\delta$ over all choices of $\D$, with $|\D| > n$, the precision is larger than $1-\epsilon$, assuming drift was detected.
\end{definition}

Notice, that the restriction to a single drift event in the definition of precision is necessary to avoid the ambiguity of which event $\hat{t_0}$ is to be compared to.

As finding the best split $\hat{t_0}$ requires the evaluation of multiple potential split points $t$, an efficient computation is important. As we have to compute the similarity at each time point, efficiency holds if the same descriptor can be used for multiple splits points, i.e. $\hat{d}(\D_-(t),\D_+(t)) = (s\circ A)(\D,t) = s(A_0(\D),t)$, where $A_0$ is independent of the split point. Using this idea we obtain the following definition:

\begin{definition}
We say that an estimator $(A,s)$ is \emph{$c$-complex}, iff $s \in \O(c)$ regarding computational complexity and $A$ factorizes as $A_0 \times \text{id}_\T$, that means it holds $(s\circ A)(\D,t) = s(A_0(\D),t)$.
\end{definition}

Notice, that the computational efficiency of $s$ crucially depends on the codomain of $A_0$. For example factorization also holds if we choose the set of all functions from $\T$ to $\R$ and $s$ as the evaluation map.

The notion of complexity restricts how much $A(\D,t)$ can be adapted to the split point $t$. Yet, the incorporation of the temporal information contained in $\D$ is desirable as it usually leads to better descriptors. We therefore say that an estimator is ``arrival time respecting'' if the descriptor uses temporal information beyond the split point: 

\begin{definition}
An estimator $(A,s)$ is \emph{arrival time respecting}, iff  
the obtained descriptor depends on the timing within the sub-windows,
i.e. it exists a split of a window $\D = \D_- \dot\cup \D_+$, and permutations $\pi_-,\pi_+$, such that $A(\D_- \cup \D_+,t) \neq A(\tilde{\D}_- \cup \tilde{\D}_+,t)$, where $\tilde{\D}_- = \{(x_i,t_{\pi_-(i)})\:|\:i=1,...,n_-\}$ is the time-permuted version of $\D_- = \{(x_i,t_i)\:|\:i=1,...,n_-\}$ and analogous for $\tilde{\D}_+$.
\end{definition}


\section{Similarity Estimators}
\label{sec:overview}
We are interested in popular instantiations of stages 2 and 3 and their properties.
Binning can be considered as one of the simplest strategies to estimate a probability distribution. Essentially, the input space is segmented and  the number of samples per bin is counted. The ratio of these samples as compared to all provides an estimate for the actual probability. Based thereon,  distance measures like total variation~\cite{DBLP:journals/corr/WebbLPG17}, Hellinger distance~\cite{hddm}, or Kullback-Leibler divergence~\cite{PCADD,kdqTree} can be computed. 
We will also consider the Jensen-Shannon metric which is based on the Kullback-Leibler divergence.
Binning on a grid was used in the work \cite{DBLP:journals/corr/WebbLPG17}, for example, to estimate the rate of change in data streams. 

Since the number of required bins grows exponentially with the number of dimensions, one might consider multiple, separate binnings in low dimensional projections for high dimensional data. Typical choices are projections onto the \emph{coordinate axis / marginals}~\cite{hddm} and onto the \emph{principal components}~\cite{PCADD}. While this strategy reduces the descriptor complexity, it is not capable of capturing drift that affects the correlation of features or of components with small variance, respectively.
As this poses a problem for drift detection in the real world, we propose a new technology: \emph{random projection} binning considers  binnings along randomly chosen projection axes. 

Instead of using an equally spaced grid structure, one can also consider a recursive splitting of the dataset similar to a decision tree with leaves forming the bins.
Depending on the way of splitting, these are  \emph{Random Trees}, where the dimension and the split point are chosen completely randomly, or \emph{$kdq$-Trees}~\cite{kdqTree}, where one successively splits the dimensions along the center. 
As such  splits often lead to slow convergence, we propose to use a comparably new alternative: \emph{Moment Trees}~\cite{momentTree}, which are designed for conditional density estimation. Here, they are trained to predict the (distribution of) time given data, i.e. $\P_{T|X}$. 
Notice, that due to the relation of a supervised problem, one can perform a parameter tuning, which is not possible for the other approaches.

Neighborhood-based approaches offer a popular and robust choice in non-parametric methods which have been widely used for 
various estimators, including Kullback-Leibler divergence~\cite{kNNDKL}.
In drift detection the \emph{Local Drift Degree (LDD)}~\cite{LDD} is one method that is explicitly based on $k$-nearest neighbors ($k$-NN).

Another, non-parametric approach are kernels. \emph{Maximum Mean Discrepancy~(MMD)}~\cite{MMD} is a kernel-based metric, which was also applied to drift detection~\cite{fail}.
These methods are summarized in Table~\ref{tab:summary}. We investigate their theoretical properties and experimental behavior in the following.

\begin{table}[t]
\newcommand{\symYes}{\ding{51}}%
\newcommand{\symProb}{(\ding{51})}%
\newcommand{\symNo}{\ding{55}}%
    \centering
    \caption{Summary of estimators (Drift Detecting: No \symNo, Drift detecting \symProb, Surely drift detecting \symYes, Arrival Time Respecting: No \symNo, Yes \symYes) }
    \begin{tabular}{llccrl}
    \hline
    Descriptor ($\H$) & Metric(s) & DD & ATR & \multicolumn{2}{r}{Complexity} \\
    \hline
    Marginal Bin.~\cite{hddm}     & \multirow{6}{*}{\parbox{3cm}{Total variation~\cite{DBLP:journals/corr/WebbLPG17}, Hellinger~\cite{hddm}, Jensen-Shannon, $D_\text{KL}$~\cite{kdqTree}}} & \symNo  & \symNo & $ \O(1)$ & cumulative histogram \\
    Random Proj. Bin. &  & \symProb  & \symNo & $\O(1)$ & cumulative histogram \\
    Random Tree (Bin.)       &  & \symYes  & \symNo & $\O(1)$ & cumulative histogram \\
    $kdq$-Tree~\cite{kdqTree} (Bin.)      &  & \symYes  & \symNo & $\O(1)$ & cumulative histogram \\
    Moment Tree (Bin.)      &  & \symYes  & \symYes & $\O(1)$ & cumulative histogram \\[0.2em]
    $k$-NN & LDD~\cite{LDD}, $D_\text{KL}$~\cite{kNNDKL} & \symYes  & \symNo & $\O(k)$ & neighborhood graph \\[0.2em]
    \parbox{3cm}{Kernel embedding of distribution~\cite{fail}} & MMD~\cite{MMD} & \symYes  & \symNo & $\O(|W|)$ & \parbox{3cm}{Cholesky decomposition of kernel matrix} \\
    \hline
    \end{tabular}
    \label{tab:summary}
\end{table}

\section{Theoretical Analysis}
\label{sec:theo}

We will now discuss some of the properties of the approaches presented in Section~\ref{sec:overview} from a theoretical point of view. We will see that, regarding  question~\ref{Q:Whether} and~\ref{Q:When}, common estimators for drift detection are  well suited. 
In the following we will always assume a drift process $(p_t,P_T)$ on $\X,\T$, with $\T \subset \R$.

\paragraph{Linear projections: } Many, in particular, simple methods use projections as a first step. However, as already discussed by \cite{PCADD} not every possible projection is also suitable. Indeed, most approaches from the literature are not:

\begin{remark}
Linear projections with respect to marginals~\cite{hddm} or principal components~\cite{PCADD} are not drift detecting, independent of the further processing. This stays true if $p_t$ is compactly supported. 
\end{remark}

Conversely, random projections are sufficient for drift detection:

\begin{theorem}
\label{thm:linear}
Let $\X = \R^d$ and
assume that $p_t$ is compactly supported, then random projection (with $w \sim \Norm(0,\I)$) with random bins is drift detecting.
\end{theorem}
\begin{proof}
All proofs can be found in the appendix.
\end{proof}
\begin{toappendix}
\subsection{A proof of Theorem~\ref{thm:linear}}
Theorem~\ref{thm:linear} is actually a direct consequence of the following lemma:
\begin{lemma}
\label{lem:proj_good}
Let $X$ and $X'$ be $\R^d$-valued, compactly supported random variables. If $\P_X \neq \P_{X'}$ then the set of projections that cannot distinguish them, i.e. $\{w \in \mathbb{S}^{d-1} \:|\: \P_{w^\top X} = \P_{w^\top X'} \}$, is a null-set (with respect to the Haar measure). 
\end{lemma}
\begin{proof}
\newcommand{\Alpha}{\boldsymbol{\alpha}}
Denote by $\mu_{\Alpha,D} = \int \prod_i x_i^{\alpha_i} \d D(x)$ the $\Alpha$th moment of $D$. Let $P = \P_X$ and $Q = \P_{X'}$. Since $P$ and $Q$ are both compactly supported, there exist moments of minimal degree $d$ which differ for $P$ and $Q$, i.e. $\exists \Alpha \in \N_0^n: d = |\Alpha| := \sum_i \alpha_i, \: \mu_{\Alpha,P}\neq \mu_{\Alpha,Q}$. Consider the function $p(w) = \int (w^\top x)^d \d P(x) - \int (w^\top x)^d \d Q(x)$. It is easy to see that $p(\lambda w) = \lambda^d p(w)$ for any $\lambda \in \R$, so in particular the zero sets of $p$ form lines through the origin. Furthermore, by multiplying out and sorting, we see that $p$ is a polynomial whose coefficients are essentially given by the difference of the moments of $P$ and $Q$ of total degree $d$. Assuming that $p(w) = 0$ for all $w$ would imply that $p = 0$ and thus all coefficients of $p$ are zero, which is the case if and only if $\mu_{\Alpha,P} = \mu_{\Alpha,Q}$ for all $\Alpha$; this is a contradiction. Thus, there has to exist a $w$ for which $p(w) \neq 0$. As the zeros of a (non-zero) polynomial always form a Lebesgue null-set and thus cannot intersect the sphere in more then a Haar null-set due to the resulting non-null-set of the corresponding rays, the statement follows.
\end{proof}
\begin{proof}[Proof of Theorem~\ref{thm:linear}]
If has no drift, then for any projection $w$ and any split point $b$ it holds $\P[w^\top X > b \:|\: T \in W_-(t)] = \P[w^\top X > b \:|\: T \in W_+(t)]$. We may compare the empirical estimates of those probabilities using a statistical test and then consider the (inverse) $p$-value minus a constant that relates to the desired threshold.

So assume $p_t$ has drift.
By \cite[Theorem 2]{conttime} there exists a $t$, such that $p_{W_-(t)} \neq p_{W_+(t)}$; it then follows from Lemma~\ref{lem:proj_good} that the set of all projections for which both measures coincide is a Haar zero set. As scaling does not change this we may sample our projection vectors from a normal distribution. 

Let $F$ and $G$ denote the cdfs of the projected $p_{W_-(t)}$ and $p_{W_+(t)}$, respectively. As $F \neq G$ there exists a $b$ such that $F(b) \neq G(b)$; we will prove, that there exists some $\epsilon > 0$ such that $F(b') \neq G(b')$ for all $b' \in [b,b+\epsilon)$ which implies the statement. Assume this is not the case, then for every $\epsilon > 0$ there exists a point $\xi$ with $b < \xi < b+\epsilon$ and $F(\xi) = G(\xi)$. We may use those $\xi$ to construct a sequence $x_n$ such that $x_n > x_{n+1} > b$, $x_n \to b$ and $F(x_n) = G(x_n)$ (start with an arbitrary $\epsilon$ to find $x_0 = \xi$, then choose $x_{n+1}$ as the $\xi$ obtained for $\epsilon = (x_n-b)/2$). Since $F$ and $G$ are cadlag, it holds that $F(b) = \lim_n F(x_n) = \lim_n G(x_n) = G(b)$, which is a contradiction. However, $[b,b+\epsilon)$ is a Lebesgue non-null-set, thus with probability larger than zero we pick a random split $b$ such that $\P[w^\top X > b \:|\: T \in W_-(t)] \neq \P[w^\top X > b \:|\: T \in W_+(t)]$. We may now proceed as in the non-drifting case. 
\end{proof}
\end{toappendix}

However, we will observe that they do not perform well for a  large  dimensionality. We conjecture that this is a consequence of the fact that they are not  arrival time respecting, and not adapted to the specific problem at hand.

\paragraph{Learneable models: } Many popular machine learning models  are also applied to estimate similarities in drift detection. Interestingly, the uniform learnability that qualifies them as valid machine learning models, also assures that the derived estimators are surely drift detecting and precise:

\begin{theorem}
\label{thm:universal}
\renewcommand{\H}{\Hypo}
Let $\X$ be a measurable space 
and let 
$\H$ be a hypothesis class of binary classifiers on $\X$.
Consider the estimators induced by
\begin{align*}
    \frac{1}{2} - \inf_{h \in \H} \E[\ell_w(h,(X,\mathbf{1}[T \in W_+(t)]))],
\end{align*}
where $\ell_w$ denotes the 0-1-loss with class reweighting, i.e. $\ell_w(y',(x,y)) = 1/\P[Y = y]$ if $y \neq y'$ and 0 otherwise.
If $\H$ is PAC-learnable, then the estimator is precise. 
If in addition, 
for all binary classification tasks on $\X$
there exists a $h \in \H$ that performs better than random,
then the estimator is also surely drift detecting. 
\end{theorem}
\begin{toappendix}
\renewcommand{\H}{\Hypo}
\subsection{A proof of Theorem~\ref{thm:universal} and Corollary~\ref{cor:Tree}}
\begin{proof}[Proof of Theorem~\ref{thm:universal}]
We first show that there exists an empirical estimator that converges (nearly) uniformly over $\T \times \H$ to $\E[\ell_w(h,\mathbf{1}[T \in W_+(t)])]$ as $|\D| \to \infty$. Then we conclude that this implies the statement. 

\textbf{Step 1: Convergence}
Interpret $\H$ as a set of functions from $\X$ to $\{-1,1\}$.
We may rewrite 
\begin{align*}
    \E[\ell_w(h,\mathbf{1}[T \in W_+(t)])] 
      &= \frac{1}{2}(\P[h(X) = 1 | T \in W_-(t)] + \P[h(X) = -1 | T \in W_+(t)])
    \\&= \frac{1}{2}\left(\E\left[\left.\frac{h(X)+1}{2}\right| T \in W_-(t)\right] - \E\left[\left.\frac{h(X)-1}{2} \right| T \in W_+(t)\right]\right)
    \\&= \frac{1}{4} \left(\E\left[\left.h(X)\right| T \in W_-(t)\right] - \E\left[\left.h(X)\right| T \in W_+(t)\right]\right) + \frac{1}{2}
\end{align*}
and then use the canonical estimator
\begin{align*}
    \frac{1}{|\D_\pm(t)|}\sum_{x \in \D_\pm(t)} h(x) \approx \E[h(x) | T \in W_\pm(t) ],
\end{align*}
for each of the expectations separately. 

Denote by $n = |\D|$, $F_T$ the cdf of $P_T$, by $\hat{F}_T(t|\D) = |\D_-(t)|/n$ its canonical estimator, and by
\begin{align*}
    \hat{\T}(\tau|\D) := 
    \left(\hat{F}_T(\cdot|\D)\right)^{-1}([\tau,1]) \setminus \left(\hat{F}_T(\cdot|\D)\right)^{-1}([1-\tau,1])
\end{align*} 
In particular, $\{\hat{F}_T(t|\D) \:|\: t \in \hat{\T}(\tau|\D)\} \subset [\tau,1-\tau]$ for all $\tau \in (0,1/2)$.
We will show that there exists a sequence $\tau_n$ such that the canonical estimator converges uniformly over $\hat{\T}(\tau_n|\D) \times \H$ and $\tau_n \to 0$ as $n \to \infty$.

Fix a $0 < \tau < 1/2$ and $\lambda > 1$. For $A_{\tau,\lambda} = \{ \{ F_T(t) \:|\: t \in \hat{\T}(\tau|\D)\} \subset [\tau/\lambda,1-\tau/\lambda] \}$ it holds
\begin{align*}
    \P(A^C_{\tau,\lambda}) 
    &= \P\left(\left\{\inf_{t \in \hat{\T}(\tau|\D)}F_T(t) < \tau/\lambda \right\} \cup \left\{\sup_{t \in \hat{\T}(\tau|\D)}F_T(t) > 1-\tau/\lambda\right\} \right)
    \\&= \P(\{F_T(\inf\hat{\T}(\tau|\D)) < \tau/\lambda  \} \cup \{F_T(\sup \hat{\T}(\tau|\D)) > 1-\tau/\lambda \} )
    \\&= \P(\{F_T(\inf\hat{\T}(\tau|\D)) - \tau < \tau/\lambda -\tau \} \\&\quad\cup \{F_T(\sup \hat{\T}(\tau|\D)) - (1-\tau) > 1-\tau/\lambda - (1-\tau) \} )
    \\&\leq \P(\{F_T(\inf\hat{\T}(\tau|\D)) - \underbrace{\hat{F}_T(\inf\hat{\T}(\tau|\D))}_{\geq \tau} < \underbrace{\tau/\lambda -\tau}_{=-\tau(1-1/\lambda)} \} \\&\quad\cup \{F_T(\sup \hat{\T}(\tau|\D)) - \underbrace{\hat{F}_T(\sup\hat{\T}(\tau|\D))}_{\leq 1-\tau} > \underbrace{1-\tau/\lambda - (1-\tau)}_{=\tau(1-1/\lambda)} \} )
    \\&\leq \P\left[ \max_{t \in \{ \inf\hat{\T}(\tau|\D)), \sup\hat{\T}(\tau|\D)) \}} |F_T(t)-\hat{F}_T(t|\D)| \geq \tau\left(1-\frac{1}{\lambda}\right) \right]
    \\&\leq \P\left[ \sup_{t \in \hat{\T}(\tau|\D)} |F_T(t)-\hat{F}_T(t|\D)| \geq \tau\left(1-\frac{1}{\lambda}\right) \right]
    \\&\leq \P\left[ \sup_{t \in \T} |F_T(t)-\hat{F}_T(t|\D)| \geq \tau\left(1-\frac{1}{\lambda}\right) \right]
    \\&\leq 2\exp(-2n \tau^2 (1-1/\lambda)^2 ),
\end{align*}
where the last inequality is the Dvoretzky-Kiefer-Wolfowitz inequality and the second equality follows from the fact that $F_T$ is monotonously increasing.

Now for all $\omega \in A_{\tau,\lambda}$, for all $t \in \hat{\T}(\tau|\D)$ it holds $\P[T \in W_\pm(t)] \geq \tau/\lambda, |\D_\pm(t)|/n \geq \tau$. Furthermore, for all $h \in \H$ it holds
\begin{align*}
    &\left|\frac{1}{|\D_\pm(t)|}\sum_{x \in \D_\pm(t)} h(x) - \E[h(X)|T \in W_\pm(t)] \right| 
    \\&\leq \left|\left(\frac{|\D_\pm(t)|}{n}\right)^{-1} \cdot \frac{1}{n} \sum_{i = 1}^n h(x_i)\I_{W_\pm(t)}(t_i) - \P[T \in W_\pm(t)]^{-1} \cdot \E[h(X)\I_{W_\pm(t)}(T)] \right| 
    \\&\leq \left|\left(\frac{|\D_\pm(t)|}{n}\right)^{-1}  - \P[T \in W_\pm(t)]^{-1} \right| \cdot \underbrace{\left| \E[h(X)\I_{W_\pm(t)}(T)] \right|}_{ \leq 1}
    \\&+ \left|\frac{1}{n} \sum_{i = 1}^n h(x_i)\I_{W_\pm(t)}(t_i) - \E[h(X)\I_{W_\pm(t)}(T)] \right| \underbrace{\left(\frac{|\D_\pm(t)|}{n}\right)^{-1}}_{\leq 1/\tau},
\end{align*}
which can be seen using the triangle and Jensen's inequality.
As $x \mapsto 1/x$ and $x \mapsto 1/(1-x)$ are Lipschitz continuous with Lipschitz constant $\lambda^2/\tau^2$ on $[\tau/\lambda,1-\tau/\lambda]$, we may upper bound the first summand by
\begin{align*}
    \left|\left(\frac{|\D_-(t)|}{n}\right)^{-1}  - \P[T \in W_-(t)]^{-1} \right| &\leq \lambda^2/\tau^2 \cdot \left|\frac{|\D_-(t)|}{n}  - \P[T \in W_-(t)] \right| & \text{ and } \\
    \left|\left(\frac{|\D_+(t)|}{n}\right)^{-1}  - \P[T \in W_+(t)]^{-1} \right| &\leq \lambda^2/\tau^2 \cdot \left|\frac{|\D_+(t)|}{n}  - \P[T \in W_+(t)] \right| \\&= \lambda^2/\tau^2 \cdot \left|\frac{|\D_-(t)|}{n}  - \P[T \in W_-(t)] \right|,
\end{align*}
using that ${|\D_+(t)|}/{n} = 1-{|\D_-(t)|}/{n}$ and $\P[T \in W_+(t)] = 1-\P[T \in W_-(t)]$. From the Dvoretzky-Kiefer-Wolfowitz it follows
\begin{align*}
    &\P\left[\sup_{t \in \hat{\T}(\tau|\D) , h \in \H} \left|\left(\frac{|\D_-(t)|}{n}\right)^{-1}  - \P[T \in W_-(t)]^{-1} \right| \right.\\&+\left. \left|\left(\frac{|\D_+(t)|}{n}\right)^{-1}  - \P[T \in W_+(t)]^{-1} \right| \geq \varepsilon, A_{\tau,\lambda} \right]
    \\&\leq \P\left[2\lambda^2/\tau^2  \sup_{t \in \T} \left|\frac{|\D_-(t)|}{n}  - \P[T \in W_-(t)]\right| \geq \varepsilon, A_{\tau,\lambda} \right]
    \\&= \P\left[\sup_{t \in \T} \left|\frac{|\D_-(t)|}{n}  - \P[T \in W_-(t)]\right| \geq 2\lambda^2/\tau^2 \varepsilon, A_{\tau,\lambda} \right]
    \\&\leq \P\left[\sup_{t \in \T} \left|\frac{|\D_-(t)|}{n}  - \P[T \in W_-(t)]\right| \geq 2\lambda^2/\tau^2 \varepsilon \right]
    \\&\leq 2\exp(-2n(2 \varepsilon \tau^2 / \lambda^2)^2)
\end{align*}

Now, consider the induced hypothesis class 
\begin{align*}
\H_-(t) &:= \{(x,s) \mapsto h(x)\I_{W_-(t')}(s) \:|\: t' \leq t \}, & \text{and}\\ \hat{\H}_-(\tau|\D) &= \H_-(\inf \hat{\T}(\tau|\D)).
\end{align*}
Notice, that for $t < t'$ we have $\H_-(t) \subset \H_-(t') \subset \H_-(1) =: \H_-$. Furthermore, for a set $D \subset \X \times \T$ with projections $D_\X,D_\T$ we have $|(\H_-(t))_{|D}| \leq |\H_{|D_\X}| \cdot (|D_\T|+1)$ for all $t$, so in particular for all $\hat{\H}_-(\tau|\D)$ independent of $\D$, up to $|\D|$, and $\tau$. 
In particular, we can upper bound the second summand using Vapnik's inequality w.r.t. $\hat{\H}_-(\tau|\D)$ by
\begin{align*}
    &\P\left[ \sup_{t \in \hat{\T}(\tau|\D) , h \in \H}\left|\frac{1}{n} \sum_{i = 1}^n h(x_i)\I_{W_-(t)}(t_i) \right.\right.\\&\qquad\qquad\qquad-\left. \E[h(X)\I_{W_-(t)}(T)] \Bigg| \P[T \in W_-(t)]^{-1} \geq \varepsilon , A_{\tau,\lambda}\right] 
    \\&\leq \P\left[ \sup_{t \in \hat{\T}(\tau|\D) , h \in \H}\left|\frac{1}{n} \sum_{i = 1}^n h(x_i)\I_{W_-(t)}(t_i) - \E[h(X)\I_{W_-(t)}(T)] \right| \geq \varepsilon \tau / \lambda , A_{\tau,\lambda}\right] 
    \\&= \P\left[ \sup_{h' \in \hat{\H}_-(\tau|\D)}\left|\frac{1}{n} \sum_{i = 1}^n (h'(x_i,t_i)+1) - \E[(h'(X,T)+1)] \right| \geq \varepsilon \tau/\lambda , A_{\tau,\lambda}\right] 
    \\&\leq \P\left[ \sup_{h' \in \hat{\H}_-(\tau|\D)}\left|\frac{1}{n} \sum_{i = 1}^n (h'(x_i,t_i)+1) - \E[(h'(X,T)+1)] \right| \geq \varepsilon \tau/\lambda \right] 
    \\&\leq 4\Gamma(2n) \exp(-n(\varepsilon \tau / \lambda)^2/72)
\end{align*}
Notice, that all those statements also hold for the analogous hypothesis class $\H_+(t)$ (defined using $W_+(t)$) with $t > t'$ and $0$ (instead of $1$) which corresponds to the approximation of $\E[h(X) | T \in W_+(t)]$.

Combining both bounds, we obtain an upper bound for the error
\begin{align*}
    & \P\left[\left.\sup_{t \in \hat{\T}(\tau|\D)}\sup_{h \in \H} \right|\E[\ell_w(h,(X,\mathbf{1}[T \in W_+(t)]))] \right.\\&\qquad\qquad\qquad-\left.\left. \frac{1}{n} \sum_{i = 1}^n \ell_w(h,(x_i,\mathbf{1}[t_i \in W_+(t)])|\D) \right| \geq \varepsilon_1+\varepsilon_2\right]
    \\&\leq \P\left[\left.\sup_{t \in \hat{\T}(\tau|\D)}\sup_{h \in \H} \right|\E[\ell_w(h,(X,\mathbf{1}[T \in W_+(t)]))] \right.\\&\qquad\qquad\qquad-\left.\left. \frac{1}{n} \sum_{i = 1}^n \ell_w(h,(x_i,\mathbf{1}[t_i \in W_+(t)])|\D) \right| \geq \varepsilon_1+\varepsilon_2, A_{\tau,\lambda} \right] + \P(A_{\tau,\lambda}^C)
    \\&\leq 2 \cdot 4\Gamma(2n) \exp(-n(\varepsilon_1 \tau / \lambda)^2/72) + 2\exp(-2n(2 \varepsilon_2 \tau^2 / \lambda^2)^2) \\&+ 2\exp(-2n \tau^2 (1-1/\lambda)^2 ).
\end{align*}
Setting $\lambda = \sqrt{\tau}, \varepsilon_1 = \sqrt{144} \cdot \tau, \varepsilon_2 = \sqrt{\tau}/2$ and assuming $0 \leq \tau \leq (3-\sqrt{5})/2$ we obtain
\begin{align*}
    &\leq 2 \cdot 4\Gamma(2n) \exp(-2n\tau^3) + 2\exp(-2 n \tau^3) + 2\exp(-2n \underbrace{\tau^2 (1-1/\sqrt{\tau})^2}_{\geq \tau^3} )
    \\&\leq 2 \cdot (4\Gamma(2n)+2) \cdot \exp(-2n\tau^3).
\end{align*}
Thus, by setting $\tau_n = (3-\sqrt{5})/(2\sqrt[4]{n})$ we have $\tau_n \to 0$, $\varepsilon_n = \sqrt{144} \tau_n+\sqrt{\tau_n}/2 \to 0$ and
\begin{align*}
    & \P\left[\sup_{t \in \hat{\T}(\tau_n|\D)}\sup_{h \in \H} \right|\E[\ell_w(h,(X,\mathbf{1}[T \in W_+(t)]))]  \\&\qquad\qquad\qquad-\left.\left. \frac{1}{n} \sum_{i = 1}^n \ell_w(h,(x_i,\mathbf{1}[t_i \in W_+(t)])|\D) \right| \geq \varepsilon_n\right] \\&\leq 2 \cdot (4\Gamma(2n)+2) \cdot \exp\left(-18 n^\frac{7}{4}\right) \to 0
\end{align*}
as $n \to \infty$ since $\log(\Gamma(2n)) \in \O(\log(n))$, due to Sauer's lemma and the fundamental theorem of learning theory. 

\textbf{Step 2: Application}
If there is no drift, then the above estimator will tend to zero. If there is drift, then there is a $t_0$ with $0 < F_T(t_0) < 1$ such that the distributions before and after differ. Since $\tau_n \to 0$ there exists an $n$ such that $t_0 \in \hat{\T}(\tau_n|\D)$ with probability at least $1-\delta/3$ so that the estimate will converge to a value larger or equal to \begin{align*}
    \sup_{h \in \H} \left| \int h(x) \d p_{W_-(t)}(x) - \int h(x) \d p_{W_+(t)}(x) \right| \overset{!}{\geq} 0
\end{align*}
and in particular eventually exceed $\varepsilon_n$ with probability at least $1-\delta/3$, assuming $\H$ provides estimators that are better than a random guess so that $>$ holds in $!$. So we detect the drift with at least $1-2/3 \delta$.

On the other hand, if there is only one drift point, then the difference $|p_{W_-(t)}(A)-p_{W_+(t)}(A)|$ is maximal at $t_0 = t$ for any $A$. It follows that the estimator is precise.
\end{proof}
\end{toappendix}

To connect this result to the existing literature, observe that (for $\X = \R^d$ and universal $\Hypo$) the estimator is equivalent to the total variation norm.
%
%
Thus, if an estimator is based on a uniform learneable model class, it is surely drift detecting and precise, but in general this requires us to retrain the model for each  split point $t$. At this point the fact that  some models do not need adaptation can increase efficiency. Indeed, for Random Trees and $kdq$-Trees we find the following statements:
\phantom{
\cite{biau2008consistency}
}

\begin{corollary}
\label{cor:Tree}
On $\X = [0,1]^d$ 
Random Trees and $kdq$-Trees with total variation norm are surely drift detecting, precise, and $\O(1)$-complex with cumulative histograms as descriptors.
\end{corollary}
\begin{toappendix}
\begin{proof}[Proof of Corollary~\ref{cor:Tree}]
We can consider a decision tree as a composition $h = v \circ l$ of a map $l : \X \to \{1,...,n\}$, assigning every point to the leaf it belongs to, and a map $v : \{1,...,n\} \to \{0,1\}$, assigning a value to every leaf. For a fixed tree structure $l$, we can obtain a hypothesis class $\H_l$ by composing $l$ with every possible value map $v$. As $\H_l$ is always finite and $l$ is constructed independently of $\D$, we can apply Theorem~\ref{thm:universal} to show the statement. Thus, it remains to show that $\H_l$ contains a model capable of detecting the drift $\P$-a.s., but since we know that both models are consistent (consider the proof of \cite[Theorem 2]{biau2008consistency} for Random Trees and notice that the proof can easily be adapted for $kdq$-Trees) and convergence in probability implies convergence a.s. along a subsequence, which is equivalent in this case as if $l'$ is a sub-tree of $l$ then $\H_{l'} \subset \H_l$, this is clear.
\end{proof}
\end{toappendix}

To obtain a similar result for Moment Trees, we make use of the fact that they can be used for conditional density estimation~\cite{momentTree}: 
The obtained tree is suitable for all classification task for the form $\mathbf{1}[T > t]$, which is exactly what is considered by Theorem~\ref{thm:universal}. We therefore conjecture that Moment Trees with total variation norm is drift detecting, precise, arrival time respecting, and $\O(1)$-complex with cumulative histograms as descriptors.

\section{Empirical Evaluation}
\label{sec:exp}

\begin{figure}[t]
    \centering
    \begin{minipage}[b]{0.46\textwidth} 
    \centering
    \includegraphics[width=\textwidth]{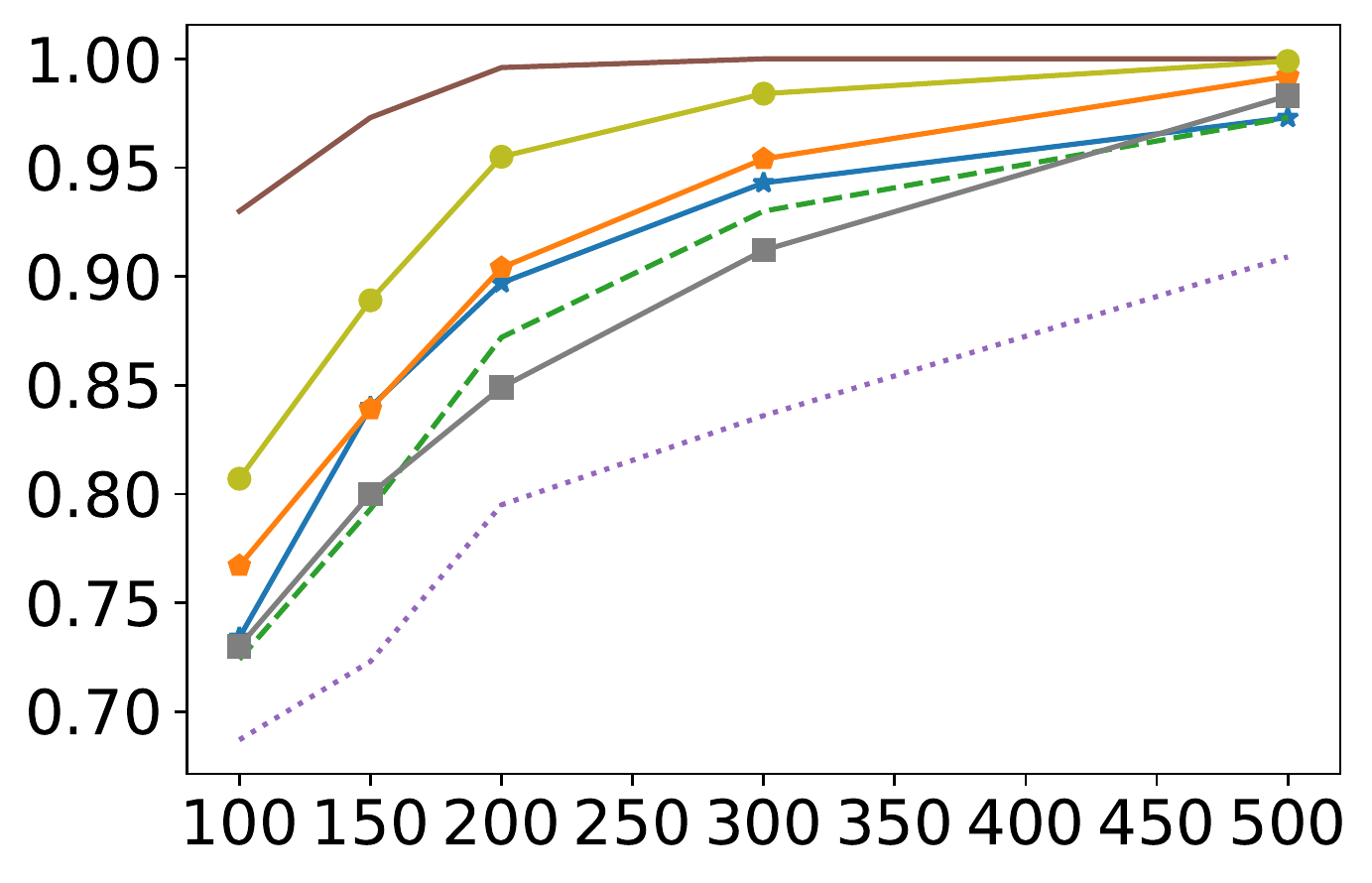}
    \subcaption{Effect of window length on $p_\text{perm}$ on weather dataset. }
    \label{fig:exp1}
    \end{minipage}
    \hfill
    \begin{minipage}[b]{0.46\textwidth}
    \centering
    \includegraphics[width=\textwidth]{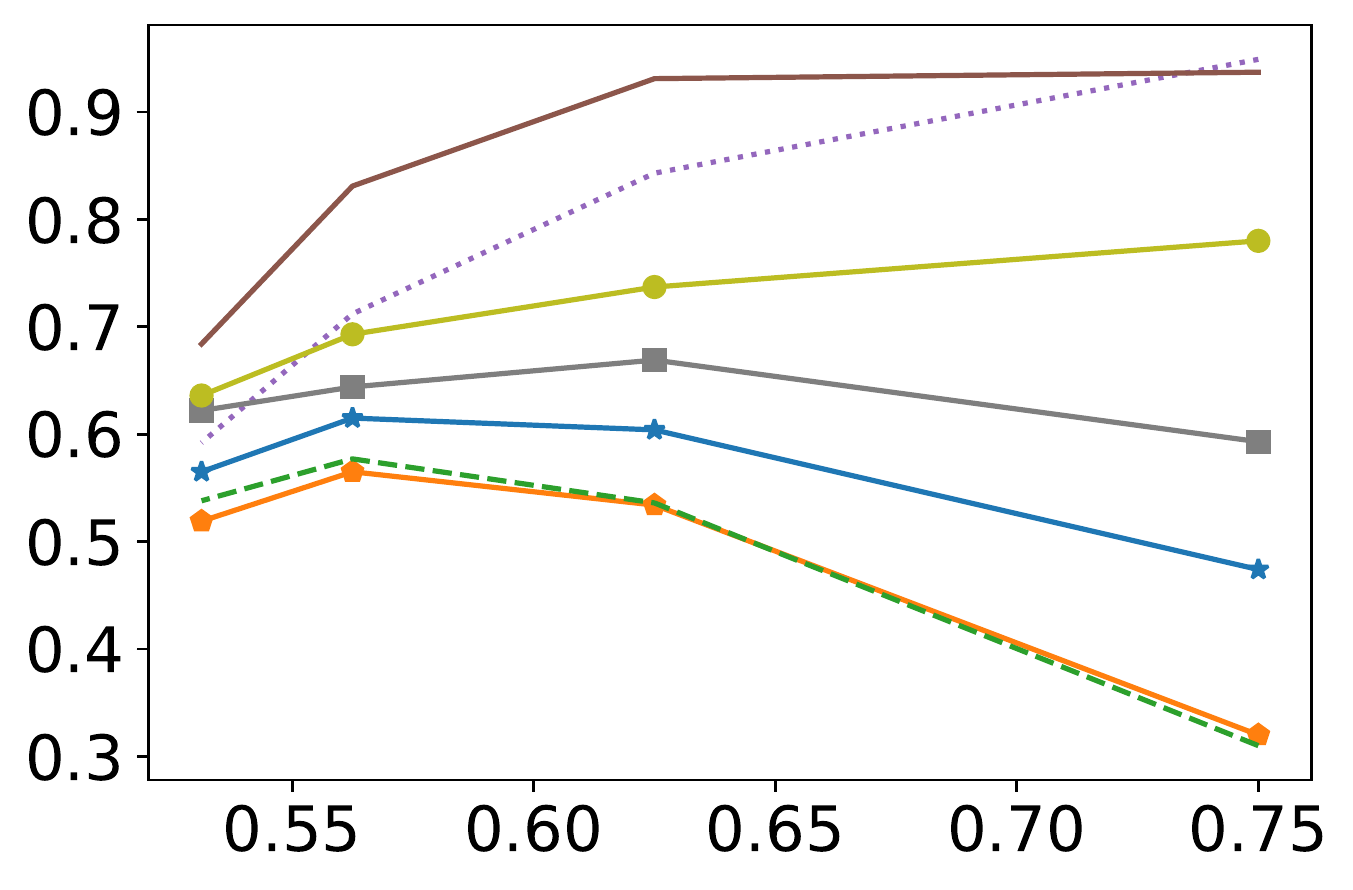}
    \subcaption{Effect of displacement ($t = t_0+\Delta$,\ $t_0=50\%$) on $p_\text{pa}(\Delta)$ on weather dataset.}
    \label{fig:exp2}
    \end{minipage}
    \\
    \begin{minipage}[b]{0.46\textwidth}
    \centering
    \includegraphics[width=\textwidth]{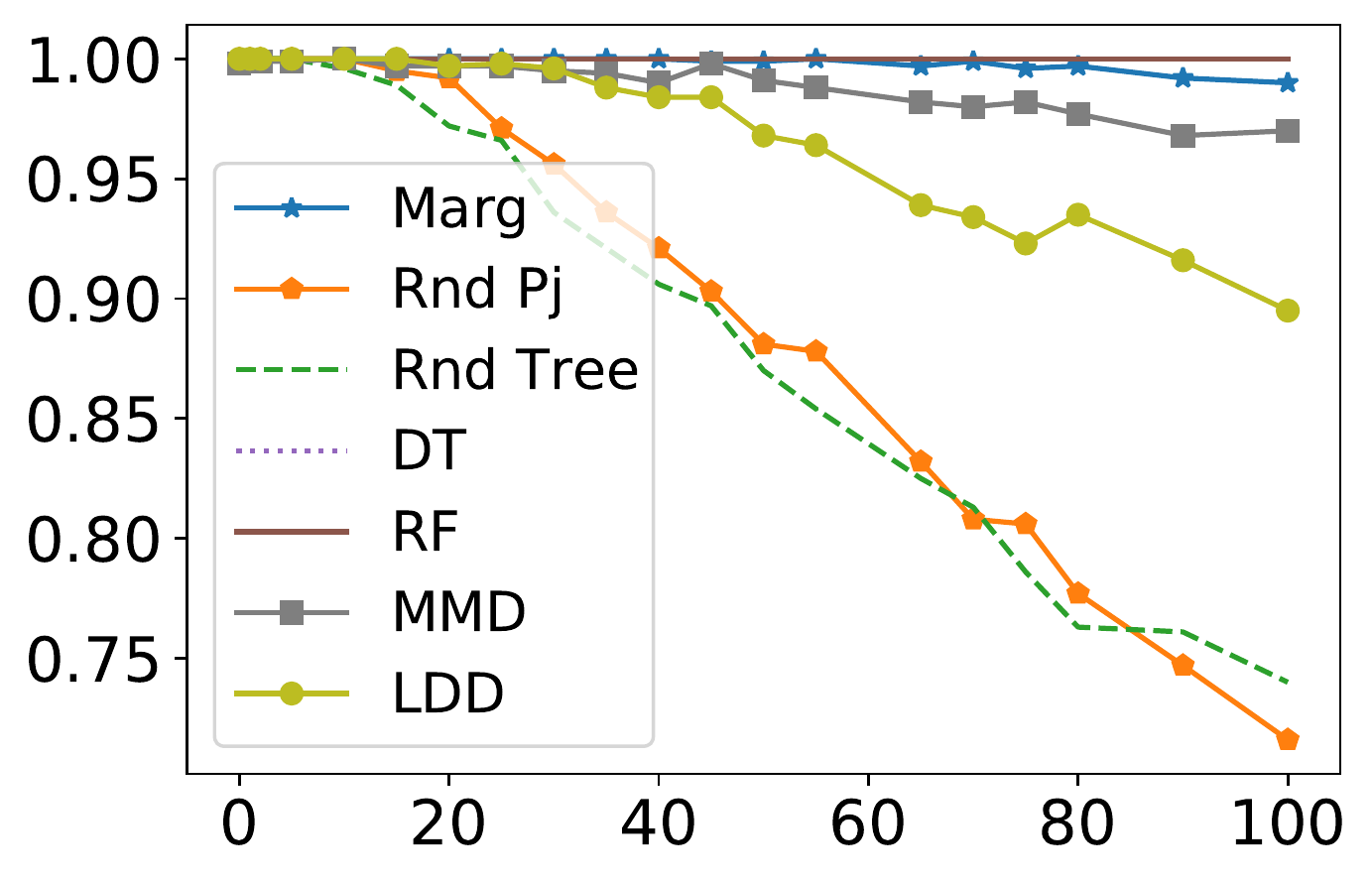}
    \subcaption{Effect of additional noise (Gaussian) dimensions on $p_\text{perm}$ on electricity dataset.}
    \label{fig:exp3}
    \end{minipage}\vfill
    \caption{Effect of parameters on statistical power ($p_\text{perm}$) and precision accuracy ($p_\text{pa}(\Delta)$). Estimators are: Marginal binning, Random Projection, Random Tree, Independent Moment Trees (DT), Random Forest Moment Trees (RF), MMD, LDD.}
    \label{fig:exp123}
\end{figure}

Based on  the theory provided in Section~\ref{sec:theo}, we can derive worst case bounds similar to standard results from classical learning theory for drift detection. 
Yet, we are also interested in average case bounds obtained from empirical estimations.

We apply the estimators as described in Table~\ref{tab:summary}.
For the binning approaches we used different numbers of bins, and equidistant and equilikely  bins. In case of Random Projection, we also vary the number of projections. In case of the $k$-NN and tree approaches we vary the number of neighbors and trees. In case of Moment Trees we consider different degrees, ensembles of independently grown Decision Trees and Random Forests. For MMD we use the biased estimator with Gauss kernel. 
Notice, that due to the setup no parameter tuning can be performed during a run.
In any case we consider all possible combinations according to Table~\ref{tab:summary}. For arrival time respecting methods we also consider skipping the last 10\% of the reference window during training. 

We use the following datasets: 
``Rotating hyperplane''~(RHP)~\cite{skmultiflow}, ``SEA''~\cite{sea}, ``stagger''~\cite{LearningWithDriftDetection}, ``RandomRBF''~(rbf)~\cite{skmultiflow}, ``Electricity Market Prices''~(elec)~\cite{electricitymarketdata}, ``Forest Covertype''~(cover)~\cite{forestcovertypedataset} and ``Nebraska Weather''~(weather)~\cite{weatherdataset}. 
For labeled datasets, the label is integrated as an additional feature, hence real drift becomes distributional drift.
To obtain a sample window with drift we sample two concepts ($\D_- \sim p_{W_-} \times U([0,1/2])$ and $\D_+ \sim p_{W_+} \times U([1/2,1])$) and concatenate them ($\D = \D_- \cup \D_+$); we then permute these samples to obtain a counterpart without drift ($\tilde\D = \{(x_i,t_{\pi(i)}) \:|\: i = 1,\cdots,n \}$ where $\D = \{(x_i,t_i)\:|\:i = 1,\cdots,n\}$). 
In case of  real world datasets we obtain two different concepts (before and after drift) by randomly sampling from before and after a given time stamp (we used a two sample test to assure that the obtained batches are indeed different, while the random selection assures no drift within the sub-windows).
Analysis of different split points on the same window use the same binning/tree; for other windows (including drift vs. no drift) we create a new binning/tree.

We investigate the effect of windows length, additional noise dimensions, offsets/imbalance (removing oldest 0\%, 12.5\%, 25\% of whole window; drift is at 50\%), and displacement of the split point $t = t_0+\Delta$ (split at $t = 50\%, 53\%, 56\%, 62\%, 75\%$ of the whole window; drift is at $t_0 = 50\%$). We repeat each experiment 1000 times. 

\paragraph{Question~\ref{Q:Whether}: ``Whether'' } 
We evaluate how well an estimator $\hat{d} = s \circ A$ detects drift. For this purpose, we estimate the probability that the estimation with drift is larger than the one without, i.e. $p_\text{perm} = \P[\hat{d}(\D_-,\D_+) > \hat{d}(\tilde\D_-,\tilde\D_+)]$, and we evaluate the probability that the estimation with and without drift can be distinguished using a threshold, i.e. $p_\text{thre} = \sup_{b} \P[\hat{d}(\D_-,\D_+) > b \geq \hat{d}(\tilde\D_-,\tilde\D_+)]$, where $\D_-, \D_+$ and $\tilde\D_-, \tilde\D_+$ are obtained 
from $\D$ and $\tilde\D$, respectively,
using the same split point $t$.
Since $p_\text{perm}$ is the probability that a random permutation increases the estimate, it is an upper bound for the statistical power ($\text{TP}/\text{T}$) of any normalization. 
Similarly, $p_\text{thre}$ is an upper bound for the balanced accuracy ($(\text{TP}/\text{T}+\text{TN}/\text{N})/2$) of (distribution dependent) threshold-based normalization.
Unlike a comparison to $0$, this procedure
does not suffer from possible estimator biases.

The results for one setup (length 150, split at drift point, no offset, total variation norm and LDD (in case of $k$-NN), where hyper-parameters are selected to optimize $p_\text{thre}$ in a previous run) are presented in Table~\ref{tab:detect}. An analysis of feature importances shows that the used descriptor has the largest impact on the results, followed by the dataset. Window length and split point displacement are in medium range, the effects of the used distance measure, and the offset are marginal. 

\begin{table}[t]
    \centering
    \caption{Empirical upper bounds: $p_\text{perm}$ (left),  $p_\text{thre}$ (right). Estimators are: Moment Tree (Random Forest), Random Projection Binning, Marginal Binning, Random Tree, MMD, and LDD.}
    \label{tab:detect}
\begin{tabular}{l@{\quad}rr@{\qquad}rr@{\qquad}rr@{\qquad}rr@{\qquad}rr@{\qquad}rr}
\toprule
dataset & \multicolumn{2}{c}{RF} & \multicolumn{2}{c}{Rnd Pj} & \multicolumn{2}{c}{Marg} & \multicolumn{2}{c}{Rnd Tree} & \multicolumn{2}{c}{MMD} & \multicolumn{2}{c}{LDD} \\
\midrule
SEA                 &                       0.56 &           0.53 &                             0.51 &                 0.49 &                    0.59 &        0.55 &                        0.62 &            0.56 &             0.54 &  0.49 &                 0.53 &     0.50 \\
cover               &                       1.00 &           0.92 &                             0.99 &                 0.88 &                    1.00 &        0.96 &                        0.99 &            0.90 &             1.00 &  0.95 &                 0.97 &     0.86 \\
elec                &                       1.00 &           1.00 &                             1.00 &                 1.00 &                    1.00 &        1.00 &                        1.00 &            1.00 &             1.00 &  0.97 &                 1.00 &     1.00 \\
rbf                 &                       1.00 &           0.99 &                             1.00 &                 0.99 &                    1.00 &        0.93 &                        1.00 &            0.96 &             0.98 &  0.90 &                 1.00 &     1.00 \\
RHP &                       0.97 &           0.86 &                             1.00 &                 0.93 &                    0.48 &        0.48 &                        1.00 &            0.93 &             0.50 &  0.52 &                 1.00 &     0.96 \\
stagger             &                       1.00 &           1.00 &                             1.00 &                 0.97 &                    1.00 &        0.96 &                        1.00 &            1.00 &             1.00 &  0.92 &                 1.00 &     0.98 \\
weather             &                       0.99 &           0.84 &                             0.86 &                 0.70 &                    0.85 &        0.69 &                        0.83 &            0.67 &             0.80 &  0.65 &                 0.91 &     0.73 \\
\bottomrule
\end{tabular}
\end{table}

As can be seen, all methods perform about equally good. Exceptions are Random Trees and Marginal Binning, which are the only methods that are better than random on the SEA dataset (Moment Tree and LDD are also able to solve SEA for larger windows sizes), and Moment Trees (RF) which is the only method that could solve the weather dataset. 
To show the impact of the window length, we plot the results for different window length for the weather dataset (see Fig.~\ref{fig:exp1}). As can be seen for most methods, using more samples increases the statistical power. 
The results on the impact of noise for the Electricity dataset are presented in Fig.~\ref{fig:exp3}.
Only Moment Trees can handle the noisy version.

\paragraph{Question~\ref{Q:When}: ``When'' } 
To evaluate precision of an estimator $\hat{d}$ we empirically evaluate the probability that the estimation at the real split point $t_0$ is larger than the one at the displaced split point $t_0+\Delta$, i.e. $p_\text{pa}(\Delta) = \P[\hat{d}(\D_-(t_0),\D_+(t_0)) > \hat{d}(\D_-(t_0+\Delta),\D_+(t_0+\Delta))]$. We refer to this as precision accuracy. Notice, that this corresponds to an ADWIN~\cite{adwin} like split point search. The feature importances provides the same results as before. The results (under the same parameters as in Table~\ref{tab:detect}) are shown in Table~\ref{tab:delay}. We also illustrate the behavior for the weather dataset in Fig.~\ref{fig:exp2} for different $\Delta$.

\begin{table}[t]
    \centering
    \caption{Precision accuracy: $\Delta = 3\%$ (left) and $\Delta = 12\%$ (right). Estimators are: Moment Tree (Random Forest), Random Projection, Marginal Binning, Random Tree, MMD, and LDD. }
    \label{tab:delay}
\begin{tabular}{l@{\quad}rr@{\qquad}rr@{\qquad}rr@{\qquad}rr@{\qquad}rr@{\qquad}rr}
\toprule
dataset & \multicolumn{2}{c}{RF} & \multicolumn{2}{c}{Rnd Pj} & \multicolumn{2}{c}{Marg} & \multicolumn{2}{c}{Rnd Tree} & \multicolumn{2}{c}{MMD} & \multicolumn{2}{c}{LDD} \\
\midrule
SEA                 &                                          0.42 &                                        0.58 &                                               0.42 &                                              0.33 &                                       0.52 &                                     0.52 &                                           0.42 &                                         0.38 &                                0.50 &                              0.48 &                                    0.50 &                                  0.53 \\
cover               &                                          0.80 &                                        0.97 &                                               0.69 &                                              0.79 &                                       0.76 &                                     0.90 &                                           0.77 &                                         0.89 &                                0.80 &                              0.93 &                                    0.71 &                                  0.83 \\
elec                &                                          1.00 &                                        1.00 &                                               0.94 &                                              1.00 &                                       0.94 &                                     1.00 &                                           0.95 &                                         1.00 &                                0.80 &                              0.90 &                                    0.95 &                                  1.00 \\
rbf                 &                                          0.96 &                                        1.00 &                                               0.93 &                                              0.99 &                                       0.81 &                                     0.91 &                                           0.89 &                                         0.97 &                                0.81 &                              0.90 &                                    0.94 &                                  1.00 \\
RHP &                                          0.73 &                                        0.93 &                                               0.75 &                                              0.90 &                                       0.48 &                                     0.48 &                                           0.76 &                                         0.90 &                                0.50 &                              0.47 &                                    0.75 &                                  0.92 \\
stagger             &                                          0.93 &                                        1.00 &                                               0.84 &                                              0.94 &                                       0.71 &                                     0.82 &                                           0.94 &                                         0.99 &                                0.82 &                              0.91 &                                    0.78 &                                  0.94 \\
weather             &                                          0.68 &                                        0.93 &                                               0.52 &                                              0.53 &                                       0.56 &                                     0.60 &                                           0.54 &                                         0.54 &                                0.62 &                              0.67 &                                    0.64 &                                  0.74 \\
\bottomrule
\end{tabular}
\end{table}

As can be seen the larger the split point displacement ($\Delta$), the higher the precision accuracy. Furthermore, except for two datasets and only with $\Delta = 3\%$, Moment Trees show the best performance. Furthermore, they tend to approach perfect precision accuracy rather quickly. 

\section{Conclusion}

In this paper we studied the theoretical and empirical properties of several metrics that are used in drift detection. We also introduced two new metric estimators based on Random Projection Binning and Moment Trees. We found that in most cases the estimation method is more important than the used distance measure, when it comes to drift detection. Also, most datasets can be solved by all methods, when it comes to drift detection. Regarding localizing the drift point, Moment Trees outperform the other methods.

\bibliographystyle{plain}
\bibliography{bib.bib}

\begin{thebibliography}{10}

\bibitem{biau2008consistency}
G{\'e}rard Biau, Luc Devroye, and G{\"a}bor Lugosi.
\newblock Consistency of random forests and other averaging classifiers.
\newblock {\em Journal of Machine Learning Research}, 9(9), 2008.

\bibitem{DBLP:journals/adt/BifetG20}
A.~Bifet and J.~Gama.
\newblock Iot data stream analytics.
\newblock {\em Ann. des T{\'{e}}l{\'{e}}comm.}, 75(9-10), 2020.

\bibitem{adwin}
Albert Bifet and Ricard Gavald{\`{a}}.
\newblock Learning from time-changing data with adaptive windowing.
\newblock In {\em Proceedings of the Seventh {SIAM} International Conference on
  Data Mining, April 26-28, 2007, Minneapolis, Minnesota, {USA}}, pages
  443--448, 2007.

\bibitem{forestcovertypedataset}
Jock~A. Blackard, Denis~J. Dean, and Charles~W. Anderson.
\newblock Covertype data set, 1998.

\bibitem{kdqTree}
Tamraparni Dasu, Shankar Krishnan, Suresh Venkatasubramanian, and Ke~Yi.
\newblock An information-theoretic approach to detecting changes in
  multi-dimensional data streams.
\newblock In {\em In Proc. Symp. on the Interface of Statistics, Computing
  Science, and Applications}. Citeseer, 2006.

\bibitem{hddm}
Gregory Ditzler and Robi Polikar.
\newblock Hellinger distance based drift detection for nonstationary
  environments.
\newblock In {\em 2011 {IEEE} Symposium on Computational Intelligence in
  Dynamic and Uncertain Environments, {CIDUE} 2011, Paris, France, April 13,
  2011}, pages 41--48, 2011.

\bibitem{DBLP:journals/cim/DitzlerRAP15}
Gregory Ditzler, Manuel Roveri, Cesare Alippi, and Robi Polikar.
\newblock Learning in nonstationary environments: {A} survey.
\newblock {\em {IEEE} Comp. Int. Mag.}, 10(4):12--25, 2015.

\bibitem{weatherdataset}
R.~{Elwell} and R.~{Polikar}.
\newblock Incremental learning of concept drift in nonstationary environments.
\newblock {\em IEEE Transactions on Neural Networks}, 22(10):1517--1531, Oct
  2011.

\bibitem{asurveyonconceptdriftadaption}
Jo\~{a}o Gama, Indr\.{e} \v{Z}liobait\.{e}, Albert Bifet, Mykola Pechenizkiy,
  and Abdelhamid Bouchachia.
\newblock A survey on concept drift adaptation.
\newblock {\em ACM Comput. Surv.}, 46(4):44:1--44:37, March 2014.

\bibitem{LearningWithDriftDetection}
Joao Gama, Pedro Medas, Gladys Castillo, and Pedro Rodrigues.
\newblock Learning with drift detection.
\newblock In {\em Brazilian symposium on artificial intelligence}, pages
  286--295. Springer, 2004.

\bibitem{MMD}
Arthur Gretton, Karsten Borgwardt, Malte Rasch, Bernhard Sch{\"o}lkopf, and
  Alex Smola.
\newblock A kernel method for the two-sample-problem.
\newblock {\em Advances in neural information processing systems}, 19, 2006.

\bibitem{electricitymarketdata}
Michael Harries and New~South Wales.
\newblock Splice-2 comparative evaluation: Electricity pricing.
\newblock 1999.

\bibitem{conttime}
Fabian Hinder, Andr{\'{e}} Artelt, and Barbara Hammer.
\newblock A probability theoretic approach to drifting data in continuous time
  domains.
\newblock {\em CoRR}, abs/1912.01969, 2019.

\bibitem{davidd}
Fabian Hinder, Andr{\'e} Artelt, and Barbara Hammer.
\newblock Towards non-parametric drift detection via dynamic adapting window
  independence drift detection ({DAWIDD}).
\newblock In Hal~Daumé III and Aarti Singh, editors, {\em Proceedings 37th
  ICML}, volume 119 of {\em PMLR}, pages 4249--4259, Virtual, 13--18 Jul 2020.
  PMLR.

\bibitem{momentTree}
Fabian Hinder, Valerie Vaquet, Johannes Brinkrolf, and Barbara Hammer.
\newblock {Fast Non-Parametric Conditional Density Estimation using Moment
  Trees}.
\newblock In {\em IEEE Computational Intelligence Magazine}. IEEE, 2021.

\bibitem{LDD}
A.~Liu, Y.~Song, G.~Zhang, and J.~Lu.
\newblock Regional concept drift detection and density synchronized drift
  adaptation.
\newblock In {\em IJCAI}, 2017.

\bibitem{Lu_2018}
Jie Lu, Anjin Liu, Fan Dong, Feng Gu, Joao Gama, and Guangquan Zhang.
\newblock Learning under concept drift: A review.
\newblock {\em IEEE Transactions on Knowledge and Data Engineering},
  31(12):2346--2363, 2018.

\bibitem{skmultiflow}
Jacob Montiel, Jesse Read, Albert Bifet, and Talel Abdessalem.
\newblock Scikit-multiflow: A multi-output streaming framework.
\newblock {\em Journal of Machine Learning Research}, 19(72):1--5, 2018.

\bibitem{kNNDKL}
Fernando P\'{e}rez-Cruz.
\newblock Estimation of information theoretic measures for continuous random
  variables.
\newblock In D.~Koller, D.~Schuurmans, Y.~Bengio, and L.~Bottou, editors, {\em
  NIPS}, volume~21. Curran Associates, Inc., 2009.

\bibitem{PCADD}
Abdulhakim~A. Qahtan, Basma Alharbi, Suojin Wang, and Xiangliang Zhang.
\newblock A pca-based change detection framework for multidimensional data
  streams: Change detection in multidimensional data streams.
\newblock In {\em Proceedings of the 21th ACM SIGKDD ICKDDM}, KDD '15, page
  935–944, New York, NY, USA, 2015. Association for Computing Machinery.

\bibitem{fail}
Stephan Rabanser, Stephan G\"{u}nnemann, and Zachary Lipton.
\newblock Failing loudly: An empirical study of methods for detecting dataset
  shift.
\newblock In H.~Wallach, H.~Larochelle, A.~Beygelzimer, F.~d\textquotesingle
  Alch\'{e}-Buc, E.~Fox, and R.~Garnett, editors, {\em Advances in Neural
  Information Processing Systems}, volume~32. Curran Associates, Inc., 2019.

\bibitem{sea}
W.~Nick Street and YongSeog Kim.
\newblock A streaming ensemble algorithm {(SEA)} for large-scale
  classification.
\newblock In {\em Proceedings of the seventh {ACM} {SIGKDD} international
  conference on Knowledge discovery and data mining, San Francisco, CA, USA,
  August 26-29, 2001}, pages 377--382, 2001.

\bibitem{DBLP:journals/widm/TabassumPFG18a}
S.~Tabassum, F.~S.~F. Pereira, S.~Fernandes, and J.~Gama.
\newblock Social network analysis: An overview.
\newblock {\em Wiley Interdiscip. Rev. Data Min. Knowl. Discov.}, 8(5), 2018.

\bibitem{DBLP:journals/corr/WebbLPG17}
Geoffrey~I. Webb, Loong~Kuan Lee, Fran{\c{c}}ois Petitjean, and Bart Goethals.
\newblock Understanding concept drift.
\newblock {\em CoRR}, abs/1704.00362, 2017.

\end{thebibliography}

\end{document}